\pdfoutput=1

\documentclass[11pt]{article}

\usepackage[final]{acl}

\usepackage{times}
\usepackage{latexsym}
\usepackage{caption}
\usepackage[T1]{fontenc}
\usepackage{enumitem}

\usepackage[utf8]{inputenc}

\usepackage{microtype}

\usepackage{inconsolata}

\usepackage{hyperref}
\usepackage{url}
\usepackage[table]{xcolor} 

\usepackage{subcaption} 
\usepackage{placeins}
\usepackage{wrapfig}
\usepackage{booktabs}
\usepackage{graphicx}
\usepackage[cmex10]{amsmath}
\usepackage{amssymb}
\usepackage{algorithmic}
\usepackage{algorithm}
\usepackage{amsthm}
\usepackage{tcolorbox}
\newtheorem{theorem}{Theorem}
\newtheorem{remark}{Remark}

\newtheorem{lemma}{Lemma}
\newtheorem{definition}{Definition}
\newtheorem{formulation}{Formulation}
\newtheorem{proposition}{Proposition}

\DeclareMathAlphabet\mathbfcal{OMS}{cmsy}{b}{n}

\newcommand{\mat}[1]{\mathbf{#1}}

\newcommand{\vect}[1]{\boldsymbol{#1}}

\title{Activation-Informed Pareto-Guided Low-Rank Compression for Efficient LLM/VLM}

\author{
\textbf{Ryan Solgi\textsuperscript{1}, Parsa Madinei\textsuperscript{1}, Jiayi Tian\textsuperscript{1},
Rupak Swaminathan\textsuperscript{2},} \\ \textbf{Jing Liu\textsuperscript{2}, Nathan Susanj\textsuperscript{2}, Zheng Zhang\textsuperscript{1}} \\
\textsuperscript{1}University of California-Santa Barbara, USA \\
\textsuperscript{2}Amazon, USA \\
\texttt{solgi@ucsb.edu}, \texttt{zhengzhang@ece.ucsb.edu}
}

\begin{document}

\maketitle

\begin{abstract}
Large language models (LLM) and vision-language models (VLM) have achieved state-of-the-art performance, but they impose significant memory and computing challenges in deployment. We present a novel low-rank compression framework to address this challenge. First, we upper bound the change of network loss via layer-wise activation-based compression errors, filling a theoretical gap in the literature. We then formulate low-rank model compression as a bi-objective optimization and prove that a single uniform tolerance yields surrogate Pareto-optimal heterogeneous ranks. Based on our theoretical insights, we propose Pareto-Guided Singular Value Decomposition (PGSVD), a zero-shot pipeline that improves activation-aware compression via Pareto-guided rank selection and alternating least-squares implementation. We apply PGSVD to both LLM and VLM, showing better accuracy at the same compression levels and inference speedup. 

\end{abstract}

\section{Introduction}
Pre-trained foundation models have achieved state-of-the-art performance in diverse domains~\citep{Vaswani2017,Bommasani2021}. However, their huge memory and computing demands pose significant barriers to efficient deployment on various platforms~\citep{Patterson2021,Strubell2019}. Hence, compressing these models has become an active area of research. Different methods have been studied, including pruning~\citep{Sanh-pruning-2020, someki2025contextaware}, distillation~\citep{Sanh-DistilBERT-2019}, quantization~\citep{Shen-qbert-2020}, and low-rank factorization~\citep{Lan-albert-2020,Hsu-wsvd-2022,Gao-awsvd-2024}. 

In this work, we investigate zero-shot compression of pre-trained models through low-rank factorization. Zero-shot compression has gained popularity due to its ability to rapidly reduce memory and computation requirements while preserving downstream performance without re-training~\citep{Frantar2022,Dettmers-spqr-2023,Frantar2023}. However layers in a pre-trained model may not exhibit a low-rank property, leading to dramatic accuracy drop in direct factorization. To address this, several studies leverage gradients or activations to guide factorization~\citep{Hsu-wsvd-2022,yuan2023asvd}. Among activation-based compression methods, SVD-LLM \cite{svdllm_wang2024, svdllmv2_wang2025} factorizes LLM layers using truncation-sensitive data whitening techniques. Despite promising results, previous methods typically apply the same compression ratio across all layers~\cite{svdllm_wang2024}, or apply heuristic rank selection methods with no theoretical guarantee~\citep{svdllmv2_wang2025, zhang2024_lr_features}. Given the large number of layers in LLMs and the cross-modality disparities in multimodal (e.g., vision–language) models~\citep{yang2024ecoflap}, choosing adaptive compression ratios between layers remains a challenge. Furthermore, without sufficient theoretical insights, some layers may be over-compressed or under-compressed, leading to overall performance degradation. 

\begin{figure}
  \centering
\includegraphics[width=0.47\textwidth]{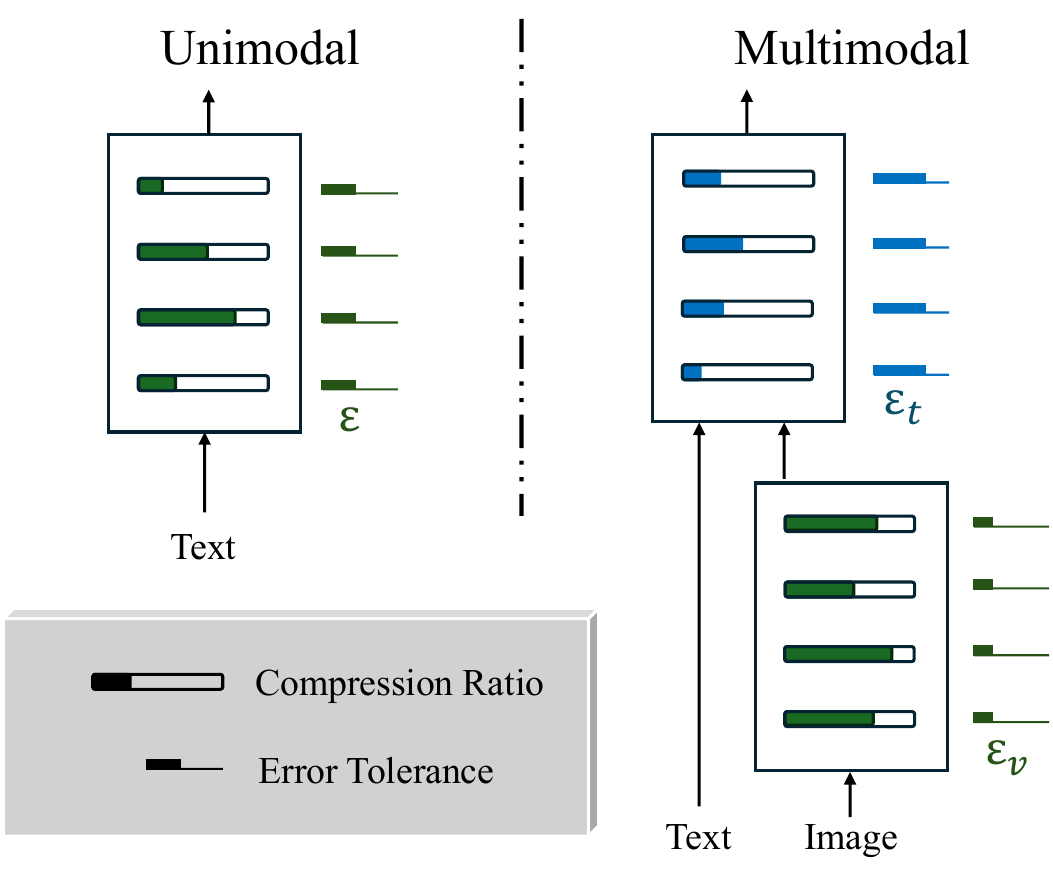}
  \setlength{\abovecaptionskip}{0pt} 
    \caption{Overview of PGSVD: (left) unimodal model using a uniform error tolerance that yields heterogeneous compression ratios; (right) multimodal model with separate uniform tolerances for each tower.}
  \label{fig:low-rank-compression}
  \vspace{-10pt}
\end{figure}

This paper tries to address two key challenges: (i) the absence of a theory that links layer-wise compression to overall model performance, and (ii) the reliance on uniform or heuristic compression ratios.
To overcome these issues, we first reveal how the errors of activation-based layer compression propagate to and impact the overall network loss. Furthermore, from a bi-objective perspective, we formally demonstrate that every uniform error-tolerance allocation is surrogate Pareto-optimal and automatically induces heterogeneous ranks across layers, collapsing a difficult search over many ranks to a single knob as illustrated in Fig.~\ref{fig:low-rank-compression}. These theoretical insights motivate us to propose Pareto-Guided Singular Value Decomposition (PGSVD) to improve the performance of low-rank LLM/VLM compression. Our main contributions can be summarized as follows: 

\begin{itemize}[leftmargin=*]
    \item \textbf{Theoretical Insights about SVD-Based LLM Compression.} We formulate compression–ratio allocation across layers as a bi-objective optimization problem. By linking network loss with layerwise activation-based compression, we rigorously show that a uniform error-tolerance allocation across layers yields a data-agnostic, surrogate-optimal architecture, defining a near Pareto-optimal trade-off between network loss and compression ratio.
    
    \item \textbf{New Compression Algorithms via Pareto-Guided Rank Selections.} We propose PGSVD, a novel zero-shot compression framework that integrates Pareto-guided rank selection with SVD to jointly achieve network-level and layer-wise optimality. PGSVD is further equipped with an alternating least squares solver for efficiently updating low-rank factors based on activations.
 
    \item \textbf{Experimental Results on LLMs and VLMs.}    We demonstrate that PGSVD outperforms prior activation-aware low-rank compression methods across multiple model sizes and datasets, and we further extend the approach to VLMs, highlighting its applicability beyond unimodal architectures. Our experiments show that PGSVD achieves more than a 30\% accuracy improvement over uniform compression ratio assignments with the same memory and inference speedup.

\end{itemize}

\section{Background} 
\paragraph{Activation-Aware Compression.} In low-rank compression of LLMs, the weight matrix of a layer denoted by $\mat{W}\in\mathbb{R}^{N \times M}$ is replaced by its low-rank approximation $\hat{\mat{W}}=\mat{A}\mat{B}$ where $\mat{A}\in \mathbb{R}^{N\times r}$ and $\mat{B}\in \mathbb{R}^{r\times M}$. Traditionally, this compression is done by computing a best low-rank approximation:
\begin{align}
\label{eq:frob_w}
    \min_{\mat{A}\in \mathbb{R}^{N\times r},\; \mat{B}\in \mathbb{R}^{r\times M}} \| \mat{W}-\mat{A}\mat{B}\|^{2}_{F}.
\end{align} 
The optimal solution is $\mat{A}^{\star}=\mat{U}_r\Sigma_r^{1/2}$ and $\mat{B}^{\star}=\Sigma_r^{1/2}\mat{V}_r^T$ where $\mat{U}_r$, $\Sigma_r$, and $\mat{V}_r$ are the truncated SVD factors corresponding to the top-$r$ singular values of $\mat{W}$. However, the weights of a pre-trained LLM/VLM may not have low-rank properties at some layers, which lead to dramatic performance loss in the above compression. In contrast, activations ($\mat{X}$) have exhibited low-rank properties~\citep{zhang2024_lr_features}. Therefore, it is more promising to shift the main objective of the compression paradigm from Eq.~\eqref{eq:frob_w} to the following:
\begin{align}
\label{eq:frob_wx}
    \min_{\mat{A}\in \mathbb{R}^{N\times r},\; \mat{B}\in \mathbb{R}^{r\times M}}\| \mat{W}\mat{X}-\mat{A}\mat{B}\mat{X}\|^{2}_{F}.
\end{align}

\paragraph{SVD-LLM.}  
A reformulation of ~\eqref{eq:frob_wx} based on whitened activations was proposed as follows:
\begin{align}
\label{eq:svd_llm}
\min_{\mat{A}\in \mathbb{R}^{N\times r},\; \mat{B}\in \mathbb{R}^{r\times M}} \quad &
\bigl\| \mat{W}\mat{X}-\mat{A}\mat{B}\mat{T}^{-1}\mat{X}\bigr\|_{F}^{2}
\end{align}
where $\mat{T}$ is Cholesky factor of $\mat{X}\mat{X}^{\top}$~\citep{svdllm_wang2024}. The optimal solution is given by 
$\mat{A}^\star = \bar{\mat{U}}_r \bar{\Sigma}_r^{1/2}$ and $\mat{B}^\star = \bar{\Sigma}_r^{1/2} \bar{\mat{V}}_r^\top$, 
where $\bar{\mat{U}}_r \bar{\Sigma}_r \bar{\mat{V}}r^\top$ is the rank-$r$ SVD of  $\mat{W}\mat{T}$.

\section{Pareto-Guided Rank Selection}
\label{sec:theory}
The performance of a compressed LLM/VLM highly depends on the choice of matrix rank per layer. For a model with many layers, this selection becomes particularly challenging because each layer admits a wide range of candidate ranks, resulting in a large discrete optimization space. This section formulates model compression as a bi-objective optimization problem which minimizes both the network loss and the total number of model parameters with respect to the layer-wise compression ratios. To solve the proposed bi-objective problem, we establish the connection between the network loss and the layer-wise approximation error, and employ a rigorous analytical approach to derive a surrogate Pareto frontier. Finally, we demonstrate that our solution yields a surrogate-optimal allocation of compression ratios throughout the network.
 
\subsection{Bi-Objective Compression Formulation}
Consider a weight matrix $\mat{W}_l\in\mathbb{R}^{N_l\times M_l}$ in layer $l$ and its rank-$r_l$ compressed weight $\hat{\mat{W}}_l^{(r_l)}$. We define the set of matrix ranks
\begin{align}
\label{eq:decision_space}
\Gamma_l := 
\Bigl\{\,r_l\in\mathbb{Z}\;\Big|\;
&e_l=\frac{\bigl\|\hat{\mat{W}}_l^{(r_l)} - \mat{W}_l\bigr\|_F}
{\bigl\|\mat{W}_l\bigr\|_F} \le \varepsilon_l 
\Bigr\}, \nonumber\\
&\text{with}\quad 0 \leq \varepsilon_l \leq 1.
\end{align}

Let $\vect{r}=[r_1, r_2, \cdots, r_L]$ specify the ranks of all $L$ layers and $|\Delta {\cal L}(\vect{r})|$ denote the absolute change in network loss due to replacing every weight matrix $\mat{W}_l$ with its approximation 
$\hat{\mat{W}}_l^{(r_l)}$. 
The total number of parameters of the compressed LLM is
\begin{align}
\label{eq:network params count}
S(\vect{r})=\sum_{l=1}^L P_l(r_l),
\end{align}
where $P_l(r_l)$ is the total number of low-rank parameters in layer $l$. 

We aim to jointly minimize the number of parameters and the loss change during compression:
\begin{tcolorbox}[colback=green!5!white]
\begin{formulation}[Bi-objective Compression]
\label{formula: biobj_formula}
\begin{align}
\min_{\vect{r}\, \in \prod_{l=1}^L \Gamma_l} 
\;\;\big( S(\vect{r}), \;\; |\Delta {\cal L}(\vect{r})| \big).
\tag{B}
\end{align}
\end{formulation}
\end{tcolorbox}

To derive a surrogate Pareto frontier of formulation~\ref{formula: biobj_formula} we first link $|\Delta {\cal L}(\vect{r})|$ to layer-wise compression errors in the following section.

\subsection{Activation-Aware Compression and Network Loss}

We show that the activation-based compression objective Eq.~\eqref{eq:frob_wx} serves as an upper bound on the network loss change. Not only does this result highlight why activation-based compression is effective, but it also connects the network loss to the layer-wise compression, which we will utilize to solve our optimization problem.

\begin{theorem}[Loss Sensitivity to Activation-Based Compression]
\label{th:loss_vs_compress}
Let $ \vect{x}_{l+1}=\sigma\, \!\bigl(\mat{W}_l \vect{x}_l\bigr) $,  
with batch $ \mat{X}_l=\bigl[\vect{x}_l^{(1)}~\cdots~\vect{x}_l^{(B)}\bigr], $  
where $\sigma$ acts elementwise and $\sup_{t\in\mathbb R}|\sigma'(t)|\le c<\infty$, and $\hat{\mat{W}}_l=\mat{W}_l+\Delta\mat{W}_l$ denote the compressed weights.  
Then, for a differentiable scalar loss $\mathcal{L}$ we have:

\begin{align}
\bigl|\Delta\mathcal{L}\bigr|
~\le~ 
G \sum_{l=1}^{L}
\Biggl(\prod_{m=l+1}^{L} \mathcal{K}_m\Biggr)
c\,\|\Delta\mat{W}_l \mat{X}_l\|_F,\nonumber
\end{align}
where  $G := \bigl\|\nabla_{\mat{Y}}\mathcal{L}\bigr\|_{F}$, $\mat{Y}=\mat{X}_{L+1}$,
$\mathcal{K}_l:=\sup_{1\le i\le B}\|\mat{J}_l^{(i)}\|$, $\mat{J}_l^{(i)}
=\operatorname{diag}\!\bigl(\sigma'(\mat{W}_l\vect{x}_l^{(i)})\bigr)\,\mat{W}_l$. 
\end{theorem}

\begin{proof}
The proof follows from a first-order perturbation analysis (Appendix~\ref{appx:loss_vs_compress_proof}).
\end{proof}

Theorem~\ref{th:loss_vs_compress} shows how compression-induced weight perturbations affect the overall loss. The local effect of layer $l$ is quantified by the activation distortion $\|\Delta \mat{W}_l \mat{X}_l\|_F$, scaled by the activation slope $c$ and the product of Jacobian norms of subsequent layers $\prod_{m=l+1}^L \mathcal{K}_m$, which captures how perturbations propagate through the network. Finally, the loss gradient $G$ converts this distortion into an upper bound on the loss change. Therefore, Theorem~\ref{th:loss_vs_compress} establishes a formal relation between network loss and the activation-aware compression objective $\|\Delta \mat{W}_l \mat{X}_l\|_F$. Although the bound still contains slack from norm inequalities and neglected higher-order terms, for a fixed pre-trained model and dataset the coefficients are constants, making it a tractable surrogate objective for compression design. The same argument applies to residual connections, softmax attention, and LayerNorm by replacing $\operatorname{diag}(\sigma')$ with the corresponding local Jacobian, which only changes the local constants $\mathcal{K}_m$ and $c$.

\subsection{Surrogate Pareto Frontier}
This subsection first proposes a scalarized surrogate, referred to as the rank allocation problem. Then we show that this problem is equivalent to a layer-wise error-tolerance allocation problem. Using this equivalence, we show that every uniform error-tolerance allocation across layers defines a point on the surrogate Pareto frontier of our bi-objective optimization. Therefore, our formulation guarantees optimal compression ratio allocations throughout the network in the surrogate sense.

We define 
\begin{align}
\alpha_l=\|\nabla_{\mat{Y}}  \mathcal{L}\|_F 
\Bigg( \prod_{m = l+1}^L \mathcal{K}_m \Bigg)c \|\mat{X}_l\|_F\|\mat{W}_l\|_F. \nonumber
\end{align} 
Let $b$ be a budget for the total number of parameters of the compressed model. 
Using Theorem~\ref{th:loss_vs_compress} and inequality $\|\Delta \mat{W}\mat{X}\|_F\le \|\Delta \mat{W}\|_F\|\mat{X}\|_F$ we have the scalarized surrogate of Formulation~\ref{formula: biobj_formula}.

\begin{tcolorbox}[colback=green!5!white]
\begin{formulation}[Rank Allocation]
\label{formula:rank_allocat}
    \begin{align}
        \min_{\vect{r}\,\in\prod_{l=1}^{L}\Gamma_l} &|\Delta \mathcal{L}(\vect{r})| \leq \sum_{l=1}^{L}\alpha_l e_l(r_l) \quad \nonumber\\
        &\text{s.t.}\quad \sum_{l=1}^{L}P_{l}(r_l)\leq b.
        \tag{P}
    \end{align}

\end{formulation}
\end{tcolorbox}
In the following, we define a mapping from rank (compression ratio) to compression error tolerance  and formally establish our error-equivalence formulation in Proposition~\ref{prop:rank_eps_equiv}.

\begin{definition}[$\varepsilon$--Parameter Mapping via SVD]
\label{def:eps_mapping}
For any matrix 
$\mat{W}\in\mathbb{R}^{N\times M}$ and tolerance $\varepsilon \in [0,1]$, 
there exists a unique minimal rank $r^\star(\varepsilon)$ such that the 
truncated SVD approximation $\hat{\mat{W}}^{(r^\star)}$ satisfies
\[
\frac{\|\mat{W}-\hat{\mat{W}}^{(r^\star)}\|_F}{\|\mat{W}\|_F} \;\le\; \varepsilon.
\]
Because $r^\star(\varepsilon)$ is minimal, it induces the 
\emph{minimal number of parameters} required to achieve error at most $\varepsilon$.  
We define the \emph{$\varepsilon$--parameter mapping}
\[
h:[0,1]\;\to\;\mathbb{Z}_{\geq 0}, 
\qquad h(\varepsilon) \;:=\; P\!\big(r^\star(\varepsilon)\big),
\]
where $P(r)=r(M+N)$ denotes the number of parameters of a rank-$r$ SVD factorization.
\end{definition}

\begin{proposition}[Rank--$\varepsilon$ Allocation Equivalence] 
\label{prop:rank_eps_equiv}
The rank-allocation problem \textnormal{(P)} and the $\varepsilon$-allocation problem
\textnormal{(E)} have the same optimal value.
\end{proposition}
\begin{proof}
    See Appendix~\ref{appx:rank_eps_equiv_proof}.
\end{proof}

\begin{tcolorbox}[colback=green!5!white]
\begin{formulation}[$\varepsilon$-Allocation]
\label{formula:error_allocate}
    \begin{align}
    \min_{\,0\le \varepsilon_1,\dots,\varepsilon_L\le 1}
    &\ \sum_{l=1}^L \alpha_l\,\varepsilon_l \quad {\rm s.t.}\quad \sum_{l=1}^L h_l(\varepsilon_l)\le b.
    \tag{E}
    \end{align}
\end{formulation}
\end{tcolorbox}

Although rank allocation can be posed as a knapsack problem (Formulation~\ref{formula:rank_allocat}), the large number of layers and candidate ranks typically makes this approach computationally intractable for modern LLMs. Similar combinatorial formulations have been explored in other compression settings and often require approximate or heuristic search strategies~\citep{frantar2022spdy,malinovskii2024higgs, sieberling2025evopress}; in our setting, low-rank compression further induces a substantially larger and more fine-grained allocation space due to the broad range of admissible ranks in each layer. Formulation~\ref{formula:error_allocate} provides an alternative continuous surrogate that is more computationally efficient and convex, while also enabling analytical characterization of the resulting allocation strategy.

Rank selection can follow two approaches. A data-driven method requires the coefficients $\alpha$; however, the activations $\mat{X}$ and therefore the coefficients $\alpha$ are unknown in practice. While they can be approximately estimated from sampled data, doing so makes the resulting network architecture data-dependent, which may reduce robustness and complicate stable on-device deployment. In contrast, a data-agnostic method adopts equal weighting across layers, following the minimax-robust surrogate (Appendix~\ref{appx:h_assumption}); while conservative, this strategy avoids data dependence, improves robustness to uncertainty, and ensures that changes in the data do not require altering the network architecture, but only updating the parameter values. We therefore adopt the latter method and show that assigning the same error tolerance ($\varepsilon$) to all layers yields heterogeneous ranks that are surrogate-optimal for the bi-objective Formulation~\ref{formula: biobj_formula}, as formally stated in Theorem~\ref{th:pareto_front}.

\begin{lemma}[Uniform $\varepsilon$ under the minimax surrogate and bounded profiles]
\label{th:uniform_eps}
Consider the minimax-robust surrogate of  \textnormal{(E)} (Remark~\ref {rem:robust_alpha_minmax}, Appendix~\ref{appx:h_assumption}), which reduces to a common weight $\alpha_l\equiv\alpha, \forall l$, and denote by
\[
p
=\min_{0\le\varepsilon_1,\dots,\varepsilon_L\le1}
\sum_{l=1}^L \alpha\,\varepsilon_l 
\quad {\rm s.t.}\quad \sum_{l=1}^L h_l(\varepsilon_l)\le b.
\]
Assume each layer’s parameter function $h_l$ is bounded by common
nonincreasing convex envelopes
$\underline h,\overline h:[0,1]\to\mathbb{R}_{\ge0}$ (see Appendix~\ref{appx:svd_profile_env}) such that 
\[
\underline h(\varepsilon)\ \le\ h_l(\varepsilon)\ \le\ \overline h(\varepsilon),
\qquad \forall\, l,\ \varepsilon\in[0,1].
\]
Let the budget $b$ satisfy $L\,\overline h(1)\le b\le L\,\overline h(0)$.  Then the $\varepsilon$-allocation problem admits a uniform solution
$\varepsilon_1=\cdots=\varepsilon_L$ that  
(i) is optimal for the symmetric surrogate using $\overline h$, with the common value 
$\varepsilon^\star$ satisfying $L\,\overline h(\varepsilon^\star)=b$;
(ii) is minimax-optimal for the worst-case admissible profiles, attained at $h_l\equiv\overline h$, and
(iii) yields bounds
\[
\alpha L\,\varepsilon^{\ell}\ \le\ p \le\ \alpha L\,\varepsilon^{\mathrm{u}},
\]
where $\varepsilon^{\ell},\varepsilon^{\mathrm{u}}$ are defined by
$L\,\underline h(\varepsilon^{\ell})=b$,  $L\,\overline h(\varepsilon^{\mathrm{u}})=b$.

\end{lemma}
\begin{proof}
See Appendix~\ref{appx:uniform_eps_proof}.
\end{proof}

\begin{theorem}[Uniform $\varepsilon$ yields the surrogate Pareto front of (B)]
\label{th:pareto_front}
Under the minimax-robust surrogate and bounded-profile condition
(Lemma~\ref{th:uniform_eps}), every uniform tolerance 
$\varepsilon \in [0,1]$ corresponds to a point on the surrogate Pareto frontier of \textnormal{(B)}, which is the set of nondominated pairs
\[
\Bigl(\sum_{l=1}^L \alpha_l e_l(r_l),\ \sum_{l=1}^L P_l(r_l)\Bigr).
\]
\end{theorem}
\begin{proof}
    See Appendix~\ref{appx:uniform_eps_pareto_proof}.
\end{proof}

Theorem~\ref{th:pareto_front} shows that every uniform layer-wise error tolerance ($\varepsilon$) for SVD compression yields a near Pareto-optimal solution for Formulation~\ref{formula: biobj_formula}. Since layer spectra differ, the same $\varepsilon$ naturally provides {\bf adaptive rank selections} across layers. Whereas previous work often applied uniform compression ratios across layers \citep{svdllm_wang2024, yuan2023asvd}, we suggest instead applying a uniform error allocation directly tied to network loss. This approach leaves a single hyperparameter $\varepsilon$ to control the loss-compression trade-off. When finer control is needed, Theorem~\ref{th:pareto_front} applies to clustered allocations, where layers are grouped into classes (e.g., MLP and attention) and a common $\varepsilon$ is assigned within each class.

\section{From Theory to Algorithms}
\paragraph{PGSVD.} Based on the theoretical results in Section~\ref{sec:theory}, we propose the PGSVD to improve the LLM/VLM performance during compression via activation-aware SVD. PGSVD first determines the ranks and initializes the factors by directly factorizing $\mat{W}$. It then further optimizes $\mat{A}$ and $\mat{B}$ to minimize Eq.~\eqref{eq:frob_wx}. Following Theorem~\ref{th:pareto_front}, PGSVD assigns a uniform compression tolerance $\varepsilon$ to all layers, which naturally leads to heterogeneous compression ratios. After determining the optimal layer-wise ranks through this allocation, PGSVD minimizes Eq.~\eqref{eq:frob_wx} with respect to the low-rank parameters using activations. 

\paragraph{Efficient ALS Implementation.} We further introduce a new alternating least squares (ALS) solver to improve the efficiency of the compression algorithm. By expanding Eq.~\eqref{eq:frob_wx} in trace form and differentiating with respect to $\mat{A}$ and $\mat{B}$ (see Appendix~\ref{appx:als_derivation}), we obtain the following ALS updates:
\begin{align}
\label{eq:als_updates}
    \mat{A}&=\mat{W}\mat{M}\mat{B}^{\top} \left( \mat{B}\mat{M}\mat{B}^{\top}\right)^{\dagger},\nonumber\\
    \mat{B} &= \left(\mat{A}^{\top}\mat{A} \right)^{\dagger}\mat{A}^{\top}\mat{W},
\end{align}
where $\mat{M}=\mat{X}\mat{X}^{T}$ is the empirical covariance matrix and $\mat{A}$ and $\mat{B}$ are initialized by the rank-$r$ approximation of $\mat{W}$ via SVD. After initialization, $\mat{A}$ and $\mat{B}$ are updated for $\tau$ iterations. Because the pseudo inverses involve only $r \times r$ matrices, these steps are computationally efficient. The algorithm~\ref{alg:svd_prals} summarizes the PGSVD steps.

\begin{algorithm}[t]
\caption{PGSVD Algorithm}
\label{alg:svd_prals}
\begin{algorithmic}
\REQUIRE{$\{\mat{W}_l\}_{l=1}^{L}$, \{$\mat{M}_l\}_{l=1}^L$, $\varepsilon$, $\tau$}
\FOR{$l=1$ \textbf{to} $L$}
\STATE $r \leftarrow  \min\{\,r\in \mathbb{Z} \mid e_l(r) \le \varepsilon\,\}$
\STATE $\mat{U}_{r}\Sigma_{r}\mat{V}_{r}^\top\leftarrow \text{SVD}(\mat{W}_l,r)$
\STATE Initialize $\mat{A}_l=\mat{U}_r\Sigma_r^{1/2}$ and  $\mat{B}_l=\Sigma_r^{1/2}\mat{V}_r^\top$
\FOR{$t=1$ \textbf{to} $\tau$} 
\STATE Update $\mat{A}_l$ and $\mat{B}_l$ via ALS [ Eq.~\eqref{eq:als_updates}]
\ENDFOR
\ENDFOR
\STATE \textbf{Return} $\{\mat{A}_l$, $\mat{B}_l\}_{l=1}^L$
\end{algorithmic}
\end{algorithm}

\paragraph{PGSVD for Multimodal Models.}
VLM with modality-specific towers (e.g., a ViT image encoder and a Transformer text encoder) exhibit different weight and gradient distributions. This induces inter-modality imbalance, making a single global compression setting systematically biased. To address this, we assign separate error tolerances $\varepsilon_v$ and $\varepsilon_t$ for the vision and text modalities, respectively. Concretely, PGSVD applies a uniform tolerance within each modality to yield heterogeneous, layer-wise ranks tailored to that modality's spectrum, then optimizes the factors via activation least squares. This two-hyperparameter design preserves a small search space while respecting modality asymmetry, producing better accuracy-efficiency trade-offs than a single global $\varepsilon$.
\begin{table*}[t]
\scriptsize
\centering
\setlength{\tabcolsep}{4pt}
\renewcommand{\arraystretch}{1.15}
\caption{Perplexity (PPL) and zero-shot accuracy (\%) for reasoning tasks at Base, 20\% and 40\% compression.}
\label{tab:language_results}
\begin{tabular}{lllc|ccccc|c}
\toprule
\textbf{Model} & \textbf{Compression} & \textbf{Method} 
& \textbf{PPL (↓)} 
& \textbf{ARC-E} & \textbf{CSQA} & \textbf{Lambada} 
& \textbf{PIQA} & \textbf{Wino} & \textbf{AvgAcc (↑)} \\
\midrule
{LLaMA-2-7B} & {} & {\color{gray}Base} & {\color{gray}5.11} & {\color{gray}76.30} & {\color{gray}33.01} & {\color{gray}68.25} & {\color{gray}78.07} & {\color{gray}69.06} & {\color{gray}64.94} \\
\cmidrule(lr){2-10}
           & 20\% & SVD-LLM   & 7.70  & 68.56 & 19.82 & 46.61 & 70.35 & 64.33 & 53.93 \\
           &      & SVD-ALS   & 7.72  & 68.73 & 20.80 & 47.18 & 70.57 & \textbf{64.80} & 54.42 \\
           &      & PGSVD & \textbf{7.38}  & \textbf{70.75} & \textbf{20.80} & \textbf{53.02} & \textbf{71.27} & 64.56 & \textbf{56.08} \\
\cmidrule(lr){2-10}
           & 40\% & SVD-LLM   & 14.95 & 50.21 & 19.49 & 16.17 & 60.94 & 58.64 & 41.09 \\
           &      & SVD-ALS   & 15.03 & 50.88 & 19.16 & 16.59 & 61.21 & 58.88 & 41.35 \\
           &      & PGSVD & \textbf{13.46} & \textbf{54.76} & \textbf{19.25} & \textbf{21.66} & \textbf{62.89} & \textbf{59.75} & \textbf{43.66} \\
\midrule
{LLaMA-2-13B} & {} & {\color{gray}Base} & {\color{gray}4.57} & {\color{gray}79.46} & {\color{gray}46.60} & {\color{gray}70.33} & {\color{gray}79.11} & {\color{gray}72.30} & {\color{gray}69.56} \\
\cmidrule(lr){2-10}
            & 20\% & SVD-LLM   & 6.17 & 71.00 & \textbf{25.23} & 57.54 & 72.91 & 67.17 & 58.77 \\
            &      & SVD-ALS   & 6.19 & 66.12 & 20.31 & 45.97 & 69.26 & 63.46 & 53.02 \\
            &      & PGSVD & \textbf{5.96} & \textbf{73.36} & 25.14 & \textbf{60.78} & \textbf{73.23} & \textbf{69.38} & \textbf{60.38} \\
\cmidrule(lr){2-10}
            & 40\% & SVD-LLM   & 10.00 & 56.61 & 19.49 & 26.47 & 63.22 & 60.85 & 45.33 \\
            &      & SVD-ALS   & 10.06 & 57.87 & \textbf{22.11} & 26.86 & 63.60 & 62.12 & 46.51 \\
            &      & PGSVD & \textbf{9.55} & \textbf{59.34} & 19.98 & \textbf{31.71} & \textbf{64.15} & \textbf{62.67} & \textbf{47.57} \\
\midrule
{Mistral-7B}  & {} & {\color{gray}Base} & {\color{gray}4.92} & {\color{gray}80.85} & {\color{gray}56.27} & {\color{gray}69.49} & {\color{gray}80.63} & {\color{gray}74.19} & {\color{gray}72.29} \\
\cmidrule(lr){2-10}
            & 20\% & SVD-LLM   & 7.06 & 69.87 & \textbf{22.03} & 50.32 & 71.60 & 65.19 & 55.80 \\
            &      & SVD-ALS   & 7.10 & 70.88 & 21.95 & 50.38 & 71.65 & 65.82 & 56.14 \\
            &      & PGSVD & \textbf{6.71} & \textbf{72.31} & 21.05 & \textbf{52.80} & \textbf{73.07} & \textbf{66.46} & \textbf{57.14} \\
\cmidrule(lr){2-10}
            & 40\% & SVD-LLM   & 16.30 & 46.30 & 19.41 & 15.93 & 59.25 & 57.14 & 39.61 \\
            &      & SVD-ALS   & 15.69 & 47.18 & 19.49 & 16.13 & 58.76 & 57.62 & 39.84 \\
            &      & PGSVD & \textbf{14.43} & \textbf{49.24} & \textbf{20.23} & \textbf{16.48} & \textbf{60.45} & \textbf{58.48} & \textbf{40.98} \\
\midrule
{LLaMA-3.1-8B} & {} & {\color{gray}Base} & {\color{gray}5.84} & {\color{gray}81.44} & {\color{gray}71.66} & {\color{gray}67.36} & {\color{gray}80.09} & {\color{gray}73.56} & {\color{gray}74.82} \\
\cmidrule(lr){2-10}
            & 20\% & SVD-LLM   & 15.51 & 59.76 & 20.07 & 35.22 & 65.83 & 61.88 & 48.55 \\
            &      & PGSVD & \textbf{14.31} & \textbf{63.68} & \textbf{23.91} & \textbf{36.81} & \textbf{68.39} & \textbf{62.12} & \textbf{50.98} \\
\bottomrule
\end{tabular}

\end{table*}

\section{Experiments}
We evaluated the performance of our proposed method on both LLM and VLM benchmarks. We also compared our method with previous activation-based low-rank compression and pruning baselines. Since our method is most relevant to SVD-LLM~\citep{svdllm_wang2024}, we first focus on comparison with SVD-LLM in Section~\ref {subsec:llm}.  Furthermore, to isolate the effect of heterogeneous rank selection, we implemented a variant of SVD-LLM by replacing the rank selection step in the PGSVD algorithm with a set of prefix ranks determined by a uniform compression ratio across all layers. We refer to this method as SVD-ALS.

We implemented our approach using the Hugging Face Transformers library and compressed all linear layers in the self-attention modules and all linear projections in the MLP blocks. To ensure numerical stability, we used full precision for model inferences when computing covariance matrices and we used double precision during compression (see Appendix~\ref{app:calibration} for calibration details). In all experiments, 10 iterations were used for ALS in both SVD-ALS and PGSVD.

\subsection{LLM Compression Results}
\label{subsec:llm}
\paragraph{Experimental Setup.} We first evaluate PGSVD against the most relevant low-rank baselines, SVD-LLM and SVD-ALS, to study the effect of Pareto-guided heterogeneous rank allocation on several LLMs, including LLaMA~\citep{touvron2023llama2, llama3_2024}, and Mistral~\citep{mistral2023mistral7b} models. Because PGSVD extends activation-based low-rank compression, our comparison focuses on methods that share the same compression framework and objective, ensuring a consistent and fair evaluation.
We report both perplexity on the WikiText-2 dataset~\citep{merity2016wikitext2} and zero-shot accuracy on a range of downstream reasoning tasks (see Appendix~\ref{app:reas_tasks}).

\begin{figure}[t]
  \centering
\includegraphics[width=\linewidth]{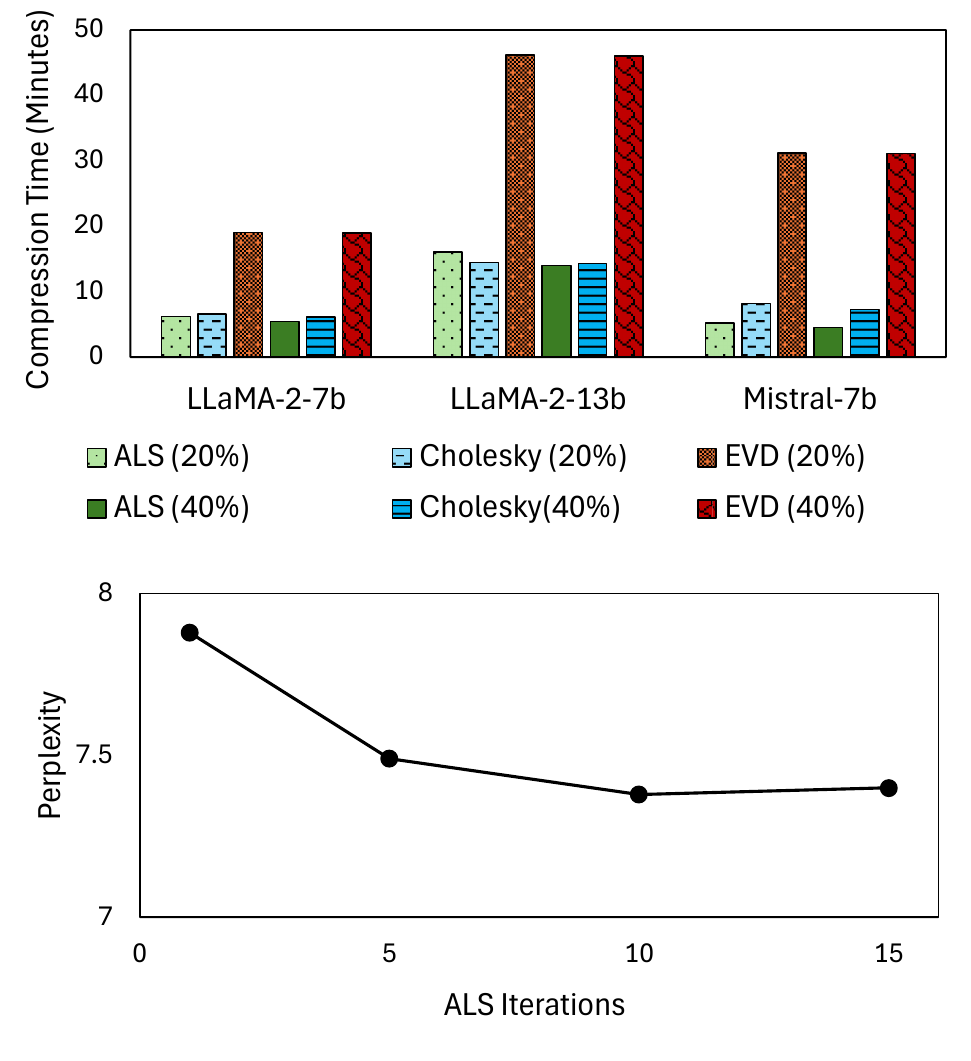}
    \setlength{\abovecaptionskip}{0pt} 
    \caption{Compression times of different solvers for different models (top) and perplexity versus the number of ALS iterations for LLaMA-2-7B (bottom).}
  \label{fig:als-ablation}
\end{figure}

\begin{figure*}[t] 
  \centering
  \begin{minipage}{0.49\linewidth}
    \centering
    \includegraphics[width=\linewidth]{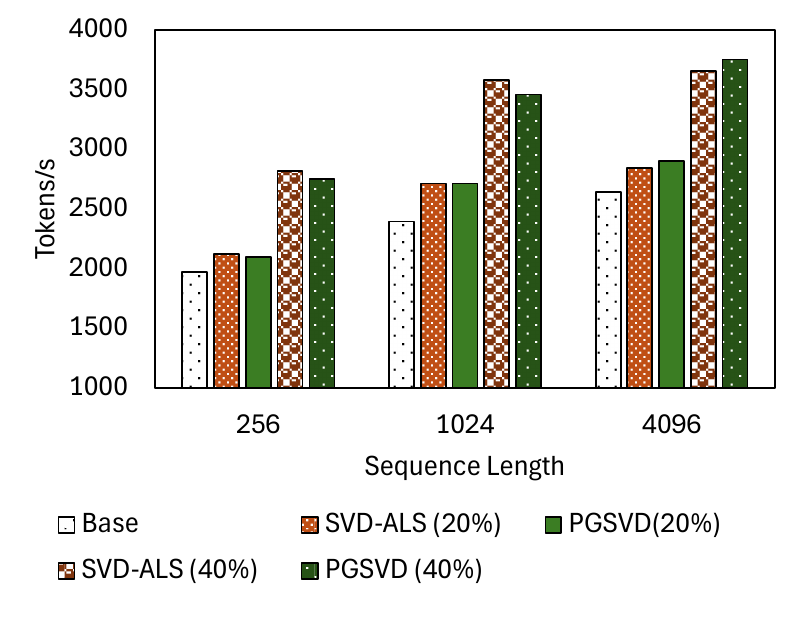}
  \end{minipage}
  \begin{minipage}{0.49\linewidth}
    \centering
    \includegraphics[width=\linewidth]{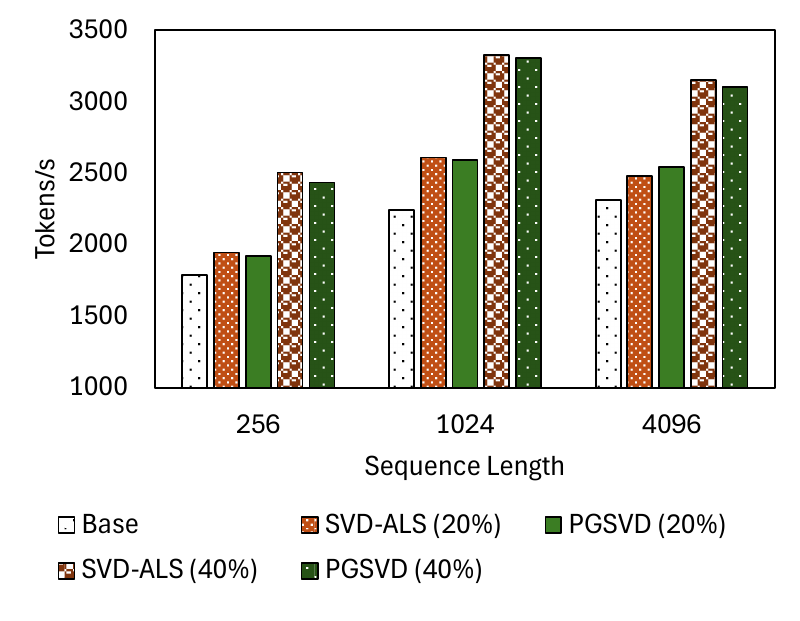}
  \end{minipage}
  \captionsetup{aboveskip=0pt}
  \caption{Inference throughput of  LLaMA-2-7b (left) and Mistral 7b (right) for 20\% and 40\% compression using PGSVD and SVD-ALS compared to the base model.}
  \label{fig:inference_latency}
\end{figure*}

\paragraph{Comparison with Activation-Aware Compression Methods.} Table~\ref{tab:language_results} shows that Pareto-guided rank selection substantially improves model performance in terms of both perplexity and accuracy on reasoning tasks, with gains of up to 30\% and an average-accuracy improvement of up to 14\%. Whereas SVD-ALS enforces a uniform compression ratio across all layers, PGSVD allocates Pareto-guided compression ratios, leading to a more balanced and effective utilization of model capacity. In addition, PGSVD outperforms SVD-LLM, achieving gains of up to 33\% on reasoning tasks and up to 6\% in average accuracy. Table~\ref{tab:language_results} further shows the superiority of PGSVD over SVD-LLM on the recent LLaMA-3.1-8B model. Note that compressed models perform near random chance (20\%) on CSQA; excluding it from the average, PGSVD remains the best method in all settings. The ALS solver eliminates numerical failures common in Cholesky decomposition of SVD-LLM and yields faster compression than eigenvalue decomposition (EVD)~\citep{svdllmv2_wang2025}, as shown in Fig.~\ref{fig:als-ablation}. We further studied the effect of the number of ALS iterations on the quality of compression in PGSVD. The results for LLaMA-2 7B are shown in Fig.~\ref{fig:als-ablation}. We observed that improvements in perplexity plateau after approximately five to ten iterations. Notably, even with a single ALS iteration, the perplexity drops to an acceptable and competitive range. 

\paragraph{Comparison with Search-Based Allocation.}
To evaluate the uniform-$\varepsilon$ rule against direct search over
heterogeneous layer-wise ranks, we compare PGSVD with a random search over
1000 sampled allocations and a genetic algorithm (GA) on LLaMA-2-7B at 40\%
compression (Table~\ref{tab:search}). For reference, PGSVD achieves a
perplexity of 13.46 in this setting (Table~\ref{tab:language_results}). None of the random allocations
outperforms PGSVD, and poor allocations degrade perplexity substantially (up
to 17.26), showing that the allocation landscape carries real, asymmetric
risk that the uniform rule avoids. A dedicated GA consuming 15 GPU-hours
improves on PGSVD by only 0.28 perplexity. Thus, the single-parameter
uniform-$\varepsilon$ rule is competitive with substantially more expensive
layer-wise search, at no search cost and with data-agnostic rank selection. PGSVD also substantially outperforms ASVD~\citep{yuan2023asvd}, which uses
sensitivity-based truncation rank searching (STRS), as shown in
Appendix~\ref{app:asvd}.

\begin{table}[t]
\caption{WikiText-2 perplexity of search-based rank allocations for LLaMA-2-7B at 40\% compression.}
\centering
\small
\begin{tabular}{lcccc}
\toprule
Method & Best & Avg. & Max & Std \\
\midrule
Random search & 14.15 & 15.63 & 17.26 & 0.55 \\
GA            & 13.18 & 14.10 & 16.52 & 0.77 \\
\midrule
\end{tabular}
\label{tab:search}
\end{table}

\paragraph{Inference Throughput.} PGSVD improves the inference throughput of the network compared to the uncompressed (base) models (Fig.~\ref{fig:inference_latency}) across different compression ratios and sequence lengths for both LLaMA-2 and Mistral models on a H100 GPU. 
Since PGSVD assigns heterogeneous compression ratios, we also compared its inference throughput with SVD-ALS, which assigns homogeneous compression ratios. With a Python implementation, PGSVD achieves an inference speedup comparable to that of SVD-ALS, without using low-level CUDA kernel optimization or specialized optimization libraries.

\begin{table*}[t]
\scriptsize
\centering
\setlength{\tabcolsep}{5pt}
\renewcommand{\arraystretch}{1.15}
\caption{Top-1 and Top-5 accuracies (Top-1/Top-5) for zero-shot compression of CLIP across six datasets.}
\label{tab:clip_results}
\begin{tabular}{llccccccc}
\toprule
Compression & Method & Caltech101 & Food101 & OxfordPets & StanfordCars & Eurosat & DTD & Average \\
\midrule
 & {\color{gray!100} Base} & {\color{gray!100} 86.39/99.02} & {\color{gray!100} 88.51/98.61} & {\color{gray!100} 89.94/97.52} & {\color{gray!100} 64.22/93.86} & {\color{gray!100} 39.00/86.03} & {\color{gray!100} 44.25/74.95} & {\color{gray!100} 68.72/91.67} \\
\midrule
20\% & SVD        & 0.63/2.97   & 1.20/5.07   & 3.76/14.45  & 0.57/2.84   & 11.12/61.40 & 1.27/10.59 & 3.09/16.22 \\
     & SVD-ALS    & 69.22/92.13 & 75.03/95.05 & 36.82/67.02 & 51.73/88.21 & 24.05/66.86 & 16.32/35.00 & 45.53/74.05 \\
     & PGSVD  & \textbf{86.61/99.20} & \textbf{84.57/97.78} & \textbf{84.06/97.17} & \textbf{56.63/90.82} & \textbf{37.24/72.31} & \textbf{26.49/46.43} & \textbf{62.60/83.95} \\
\midrule
40\% & SVD              & 0.63/3.29  & 0.90/5.00  & 2.92/13.30 & 0.67/2.38  & 11.10/52.56 & 3.46/10.79 & 3.28/14.55 \\
     & SVD-ALS          & 73.35/94.53 & 69.81/92.53 & 7.52/26.62 & 31.94/71.88 & 20.69/69.37 & 19.15/40.69 & 37.08/65.94 \\
     & PGSVD & \textbf{76.95/95.67} & \textbf{72.48/93.33} & \textbf{55.68/83.02} & \textbf{39.77/78.70} & \textbf{38.95/69.57} & \textbf{21.49/36.80} & \textbf{50.89/76.18} \\
\bottomrule
\end{tabular}

\end{table*}

\subsection{VLM Compression Results}
\label{subsec:vlm}
\paragraph{Experimental Setup.} We further evaluate the performance of PGSVD for zero-shot compression of the CLIP model~\citep{radford2021clip} on several standard benchmarks, including Caltech101~\citep{fei2004learning}, Food101~\citep{bossard2014food101}, OxfordPets~\citep{parkhi2012catsdogs}, StanfordCars~\citep{krause2013stanfordcars}, EuroSAT~\citep{helber2019eurosat}, and DTD~\citep{cimpoi2014dtd}. In this experiment, we include traditional SVD, activation-aware SVD (SVD-ALS), and our Pareto-guided variant (PGSVD) to examine both the effect of activation-based compression and the benefit of rank allocation. We use SVD rather than SVD-LLM as a baseline because activation-based compression has not previously been applied to vision-language models, and SVD provides a direct reference for isolating the improvement introduced by activation-based compression.

\paragraph{Background.} The successful compression of VLMs has remained challenging due to their architectural complexity and cross-modality disparities~\citep{ye2023mplugowl,li2023blip2,shi2023upop}. Recently, ECoFLaP~\citep{yang2024ecoflap} achieved 40–60\% unstructured sparsity in VLMs with minimal accuracy loss. However, because this sparsity is both relatively low and unstructured, it does not translate into proportional real-time memory savings on standard platforms while low-rank compression can achieve actual memory reduction.

\paragraph{Performance Evaluation.} Table~\ref{tab:clip_results} lists the Top-1 and Top-5 accuracies of the CLIP model compressed with SVD, SVD-ALS, and PGSVD, compared with the uncompressed (base) model. Naïve SVD reduces accuracies in almost all benchmarks to near zero. SVD-ALS, significantly improves the accuracies since activation is considered in compression. PGSVD achieves the best accuracy across all datasets, closing the gap with the base model. In most datasets, we observe greater sensitivity to the vision model than to the language model. Therefore, different tolerances are assigned to different models, but within each model all layers share a uniform error tolerance, leading to heterogeneous compression ratios across layers. By adjusting only two hyperparameters (one for each modality), PGSVD enables an effective allocation of compression ratios for all layers in the network.

\subsection{Comparison with Other Baseline Methods}

\begin{table}[t]
  \scriptsize
  \centering
  \setlength{\tabcolsep}{6pt}
  \renewcommand{\arraystretch}{1.15}
    \caption{Comparison of PGSVD with other baselines on the WikiText-2 dataset.}
  \label{tab:pruning_comparison}
  \begin{tabular}{llccc}
    \toprule
    \textbf{Model} & \textbf{Methods} & \multicolumn{3}{c}{\textbf{Compression}} \\
    \cmidrule(lr){3-5}
    & & \textbf{10\%} & \textbf{30\%} & \textbf{50\%} \\
    \midrule
    LLaMA-2-7B 
      & PGSVD     & \textbf{6.52} & \textbf{9.20} & 27.46 \\
      (PPL = 5.11)
      & LLM-Pruner    & 7.11 & 13.56 & 31.05 \\
      & SliceGPT      & 6.69 & 11.94 & \textbf{25.84} \\
      & ShortGPT      & 6.98 & 33.21 & 268.11 \\
    \addlinespace[2pt]
    LLaMA-2-13B
      & PGSVD     & \textbf{5.36} & \textbf{7.09} & 18.04 \\
      (PPL = 4.57)
      & LLM-Pruner    & 5.57 & 12.19 & 32.20 \\
      & SliceGPT      & 5.88 & 9.97 & \textbf{10.77} \\
      & ShortGPT      & 5.40 & 30.48 & 187.23 \\
    \bottomrule
  \end{tabular}
\end{table}

We compare PGSVD with other baseline methods including LLM-Pruner~\citep{ma2023llmpruner}, ShortGPT~\citep{men2024shortgpt}, and SliceGPT~\citep{ashkboos2024slicegpt}. ShortGPT is a zero-shot compression method, whereas LLM-Pruner 
leverages gradient information for compression and should therefore be viewed as a stronger baseline. PGSVD, by contrast, performs compression without fine-tuning.

\paragraph{Language Modeling.} As shown in Table~\ref{tab:pruning_comparison}, despite not using gradient information and relying only on a small sample of data for covariance approximation, PGSVD consistently achieves lower perplexity than   LLM-Pruner and ShortGPT. Compared to SliceGPT~\citep{ashkboos2024slicegpt}, PGSVD performs better at lower compression ratios, but is outperformed at the highest compression level (50\%).

\begin{table}
\scriptsize
\centering
\caption{Reasoning performance of PGSVD versus pruning methods for 20\% compression of LLaMA-2-7B.}
\label{tab:prunning_acc}
\begin{tabular}{lcccc}
\toprule
\textbf{Task} & {\color{gray}\textbf{Base}} & \textbf{LLM-Pruner} & \textbf{SliceGPT} & \textbf{PGSVD} \\
\midrule
PIQA        & {\color{gray}78.07} & \textbf{75.95} & 61.26 & 71.27 \\
WinoGrande  & {\color{gray}69.06} & 63.38 & 59.83 & \textbf{64.56} \\
HellaSwag   & {\color{gray}76.00} & \textbf{67.83} & 44.28 & 60.96 \\
ARC-e       & {\color{gray}76.30} & 64.31 & 46.09 & \textbf{70.75} \\
ARC-c       & {\color{gray}43.34} & \textbf{39.93} & 28.41 & 36.52 \\
Average     & {\color{gray}68.55} & \textbf{62.28} & 47.97 & 60.81 \\
\bottomrule
\end{tabular}
\end{table}

\paragraph{Reasoning Tasks.} We also evaluated the accuracy of these methods on multiple reasoning benchmarks, as reported in Table~\ref{tab:prunning_acc}. Despite being applied in a zero-shot manner without any task-specific fine-tuning, PGSVD maintains competitive or superior accuracy compared to LLM-Pruner on most tasks, particularly on ARC-e and WinoGrande, highlighting its robustness for reasoning-oriented evaluations. Although SliceGPT achieves stronger results on perplexity at extreme pruning ratios, its reasoning performance drops sharply at 20\% compression, whereas PGSVD retains a balanced trade-off between compression and reasoning accuracy. Overall, these results confirm that PGSVD preserves both generalization and reasoning capabilities.

\section{Conclusion}

We have presented a framework for the compression of LLMs/VLMs that integrates theory and practice. We have established a loss bound showing how layer compression influences overall network loss. By formulating rank selection as a bi-objective optimization problem, we have theoretically demonstrated that a uniform error assignment across layers yields surrogate Pareto-optimal heterogeneous ranks, simplifying rank search to a single knob. We have further proposed PGSVD to improve activation-aware compression via Pareto-guided rank selection. Empirically, PGSVD has shown consistent performance improvement in various models, increasing accuracy by over 30\% on LLM reasoning tasks. PGSVD has also been effectively generalized to vision--language modeling,
achieving up to 40\% compression in zero-shot settings while avoiding the
significant accuracy drop observed in baselines.

\section{Limitations}
The PGSVD algorithm relies on uniform-tolerance policies derived from robust assumptions. Future work may investigate learning tolerance thresholds directly from data, while further extensions involve coupling the framework with quantization and lightweight fine-tuning. Extending PGSVD to larger VLMs is another promising future direction.
\bibliography{main}

\clearpage

\appendix
\raggedbottom
\section{Proof of Theorem~\ref{th:loss_vs_compress}}
\label{appx:loss_vs_compress_proof}

\begin{proof}
Let $\|\cdot\|$ be the operator (spectral) norm and $\|\cdot\|_F$ the Frobenius norm. 
For batch size $B$, let 
$\mat{X}_l=[\vect{x}_l^{(1)}\,\cdots\,\vect{x}_l^{(B)}]$ and define:
\begin{align}
\vect{z}_l^{(i)} &= \mat{W}_l \vect{x}_l^{(i)}, \nonumber\\
\mat{D}_l^{(i)} &:= 
\operatorname{diag}\!\big(\sigma'(\vect{z}_l^{(i)})\big), \nonumber\\
\mat{J}_l^{(i)} &:= \mat{D}_l^{(i)} \mat{W}_l. \nonumber
\end{align}

Assume $\Delta\mat{X}_1=0$. Using the first-order expansion (dropping second-order terms) for each sample $i$ we have:
\begin{align}
\Delta\vect{x}_{l+1}^{(i)}
\;\approx\;
\mat{D}_l^{(i)}\!\big(\mat{W}_l \Delta\vect{x}_l^{(i)} 
+ \Delta\mat{W}_l \vect{x}_l^{(i)}\big). \nonumber
\end{align}

Assume $\|\mat{D}_l^{(i)}\|\le c$, then we have:
\begin{align}
\|\Delta\vect{x}_{l+1}^{(i)}\|_2 
&\le 
\|\mat{J}_l^{(i)}\|\,\|\Delta\vect{x}_l^{(i)}\|_2 
+ c\,\|\Delta\mat{W}_l \vect{x}_l^{(i)}\|_2. \nonumber
\end{align}

Stack columns and define:
\begin{align}
\mat{G}_l &:= 
\big[\mat{J}_l^{(1)}\Delta\vect{x}_l^{(1)}
\cdots\mat{J}_l^{(B)}\Delta\vect{x}_l^{(B)}\big], \nonumber\\
\mat{Q}_l &:= 
\big[\mat{D}_l^{(1)}\Delta\mat{W}_l\vect{x}_l^{(1)}
\cdots\mat{D}_l^{(B)}\Delta\mat{W}_l\vect{x}_l^{(B)}\big]. \nonumber
\end{align}

Then $\Delta\mat{X}_{l+1} \approx \mat{G}_l + \mat{Q}_l$ and:
\begin{align}
\|\mat{G}_l\|_F^2 
&= \sum_{i=1}^B 
\|\mat{J}_l^{(i)}\Delta\vect{x}_l^{(i)}\|_2^2\nonumber\\
&\le\; 
\Big(\sup_i \|\mat{J}_l^{(i)}\|^2\Big)
\sum_{i=1}^B \|\Delta\vect{x}_l^{(i)}\|_2^2 
\nonumber\\
&= 
\Big(\sup_i \|\mat{J}_l^{(i)}\|\Big)^2 
\|\Delta\mat{X}_l\|_F^2, 
\nonumber\\[2pt]
\|\mat{Q}_l\|_F^2 
&= \sum_{i=1}^B 
\|\mat{D}_l^{(i)}\Delta\mat{W}_l\vect{x}_l^{(i)}\|_2^2\nonumber\\
&\le\; 
c^2 \sum_{i=1}^B \|\Delta\mat{W}_l\vect{x}_l^{(i)}\|_2^2 
\nonumber\\
&= c^2 \|\Delta\mat{W}_l \mat{X}_l\|_F^2. \nonumber
\end{align}

By triangle inequality in Frobenius norm we have:
\begin{align}
\|\Delta\mat{X}_{l+1}\|_F 
&\le 
\|\mat{G}_l\|_F + \|\mat{Q}_l\|_F 
\nonumber\\
&\le 
\mathcal{K}_l\|\Delta\mat{X}_l\|_F 
+ c\,\|\Delta\mat{W}_l \mat{X}_l\|_F, \nonumber
\end{align}

where $\mathcal{K}_l:=\sup_i \|\mat{J}_l^{(i)}\|$ and unrolling for $l=1,\dots,L$ we have:
\begin{align}
\|\Delta\mat{X}_{L+1}\|_F
\;\le\;
\sum_{l=1}^L 
\Biggl(\prod_{m=l+1}^{L} \mathcal{K}_m\Biggr)
c\,\|\Delta\mat{W}_l \mat{X}_l\|_F. \nonumber
\end{align}

For a scalar loss $\mathcal{L}(\mat{Y})$ with 
$\mat{Y}=\mat{X}_{L+1}$, first-order Taylor and Cauchy–Schwarz yield:
\begin{align}
\Delta \mathcal{L} 
\;&\approx\; 
\langle \nabla_{\mat{Y}} \mathcal{L},\, \Delta\mat{Y}\rangle_F\nonumber\\
&\Rightarrow
|\Delta \mathcal{L}| 
\;\le\; 
\|\nabla_{\mat{Y}} \mathcal{L}\|_F \,\|\Delta\mat{Y}\|_F. \nonumber
\end{align}
\end{proof}

\section{Proof of Proposition~\ref{prop:rank_eps_equiv}}
\label{appx:rank_eps_equiv_proof}
\begin{proof}
(P $\Rightarrow$ E) 
Take any feasible ranks $\{r_l\}_{l=1}^L$ for (P). Set $\varepsilon_l:=e_l(r_l)$. 
Since $h_l(\varepsilon_l)$ is the \emph{minimal} parameter count achieving error 
$\le \varepsilon_l$, we have $h_l(\varepsilon_l)\le P_l(r_l)$ for each $l$. 
Thus $\sum_l h_l(\varepsilon_l)\le \sum_l P_l(r_l)\le b$, so 
$\{\varepsilon_l\}$ is feasible for (E), and the (E) objective equals 
$\sum_l \alpha_l\,\varepsilon_l=\sum_l \alpha_l\,e_l(r_l)$, the surrogate objective in (P).
Taking the infimum over feasible $\{r_l\}$ gives $\mathrm{OPT}_{\text{(E)}}\le \mathrm{OPT}_{\text{(P)}}$.

(E $\Rightarrow$ P)
Let $\{\varepsilon_l^\star\}$ be optimal for (E). For each $l$, pick
$r_l^\star\in\arg\min\{P_l(r): e_l(r)\le \varepsilon_l^\star\}$, so that
$P_l(r_l^\star)=h_l(\varepsilon_l^\star)$ and $e_l(r_l^\star)\le \varepsilon_l^\star$.
Then $\sum_l P_l(r_l^\star)=\sum_l h_l(\varepsilon_l^\star)\le b$, so $\{r_l^\star\}$ is
feasible for (P), and
\[
\mathrm{OPT}_{\text{(P)}}=\sum_{l=1}^L \alpha_l\,e_l(r_l^\star)\ \le\ \sum_{l=1}^L \alpha_l\,\varepsilon_l^\star
=\mathrm{OPT}_{\text{(E)}}.
\]
Hence $\mathrm{OPT}_{\text{(P)}}\le \mathrm{OPT}_{\text{(E)}}$. Combining both directions yields equality.
\end{proof}

\section{Proof of Lemma~\ref{th:uniform_eps}}
\label{appx:uniform_eps_proof}
\begin{proof}
Write $\sum_l(\cdot)$ for $\sum_{l=1}^L(\cdot)$.
By assumption, $\underline h,\overline h:[0,1]\to\mathbb{R}_{\ge0}$ are
nonincreasing, convex, and bound each $h_l$ pointwise. The budget range
$L\,\overline h(1)\le b\le L\,\overline h(0)$ ensures feasibility.

\smallskip\noindent
\emph{(i) Symmetric surrogate $\overline h$.}
Consider
$\min_{\varepsilon\in[0,1]^L}\sum_l\varepsilon_l$
s.t. $\sum_l\overline h(\varepsilon_l)\le b$.
This program is convex; Slater holds for interior $b$, hence KKT are
necessary and sufficient. Let $\mu\ge0$ be the multiplier for the
coupling constraint, $u_l,v_l\ge0$ for $0\le\varepsilon_l\le1$, and
$g_l\in\partial\overline h(\varepsilon_l)$. KKT:
(i) stationarity $1+\mu g_l-u_l+v_l=0$,
(ii) complementary slackness $\mu\big(\sum_l\overline h(\varepsilon_l)-b\big)=0$,
$u_l\varepsilon_l=0$, $v_l(\varepsilon_l-1)=0$,
(iii) primal/dual feasibility.
Since $\overline h$ is strictly decreasing and $b$ is nondegenerate,
the coupling constraint is active, so $\mu>0$.
If $0<\varepsilon_l<1$ then $u_l=v_l=0$ and
$1+\mu g_l=0$. Monotonicity of $\partial\overline h$ (convexity) forces
all interior $\varepsilon_l$ to be equal. If some coordinates were at
the boundary while others interior, stationarity with $\mu>0$ and the
monotonicity of $\partial\overline h$ would conflict with the active
budget unless all $\varepsilon_l$ coincide (up to kinks where subgradients
are intervals). Hence there is a uniform optimizer
$\varepsilon_l\equiv\varepsilon^\star$ determined by
$L\,\overline h(\varepsilon^\star)=b$ (unique on the active segment by
continuity and strict decrease).

\smallskip\noindent
\emph{(ii) Minimax optimality.}
Let $\mathcal H$ be the admissible families with
$\underline h\le h_l\le\overline h$ pointwise. For any
$\{h_l\}\in\mathcal H$ and any $\varepsilon$,
$\sum_l h_l(\varepsilon_l)\le\sum_l\overline h(\varepsilon_l)$, hence
$\{\varepsilon:\sum_l\overline h(\varepsilon_l)\le b\}
\subseteq\{\varepsilon:\sum_l h_l(\varepsilon_l)\le b\}$.
Minimizing the same linear objective over a smaller set yields a larger
(optimal) value, so the worst case is attained at $h_l\equiv\overline h$.
By (i), its optimizer is uniform. Thus the uniform solution is minimax.

\smallskip\noindent
\emph{(iii) Bracketing.}
Define $\varepsilon^{\mathrm u},\varepsilon^{\ell}\in[0,1]$ by
$L\,\overline h(\varepsilon^{\mathrm u})=b$ and
$L\,\underline h(\varepsilon^{\ell})=b$
(use generalized inverses at flats). The uniform vector
$(\varepsilon^{\mathrm u},\dots,\varepsilon^{\mathrm u})$ is feasible
for the true problem since $h_l\le\overline h$, hence
$p\le \alpha L\,\varepsilon^{\mathrm u}$. For the lower
bound, replace each $h_l$ by $\underline h$, which enlarges the feasible
set; the relaxed problem is convex and, by the same KKT argument, has a
uniform optimum at $\varepsilon^{\ell}$, yielding
$\alpha L\,\varepsilon^{\ell}\le p$. Combining gives
$\alpha L\,\varepsilon^{\ell}\le p\le \alpha L\,\varepsilon^{\mathrm u}$.
\end{proof}

\section{Minimax-Robust Surrogate for Unknown Sensitivities}
\label{appx:h_assumption}
\begin{remark}[Minimax-robust surrogate for unknown sensitivities]
\label{rem:robust_alpha_minmax}
Recall problem \textnormal{(E)} defined as $\min_{0\le \varepsilon_l\le1}\sum_{l}\alpha_l\varepsilon_l$
s.t. $\sum_{l}h_l(\varepsilon_l)\le b$. 
When $\alpha_l$ are uncertain, consider the robust counterpart
\[
\min_{\varepsilon}\max_{\alpha\in\mathcal U}\sum_{l}\alpha_l\varepsilon_l,
\quad 
\mathcal U=\{\alpha:\underline\alpha_l\le\alpha_l\le\overline\alpha_l\}.
\]
For any feasible $\varepsilon$, 
$\max_{\alpha\in\mathcal U}\sum_{l}\alpha_l\varepsilon_l
=\sum_{l}\overline\alpha_l\varepsilon_l$ since $\varepsilon_l\!\ge\!0$.
If the bounds are common ($\overline\alpha_l=\overline\alpha$), this reduces
exactly to the homogeneous-$\alpha$ objective 
$\overline\alpha\sum_{l}\varepsilon_l$. 
If they differ, letting $\bar\alpha=\max_l\overline\alpha_l$ gives
$\sum_{l}\overline\alpha_l\varepsilon_l\le
\bar\alpha\sum_{l}\varepsilon_l$, so minimizing
$\bar\alpha\sum_{l}\varepsilon_l$ is \emph{minimax-safe}.
Thus, $\alpha_l\equiv\bar\alpha$ provides a principled robust surrogate for \textnormal{(E)}.
\end{remark}

\section{Proof of Theorem~\ref{th:pareto_front}}
\label{appx:uniform_eps_pareto_proof}
\begin{proof}
Fix a uniform tolerance $\varepsilon\in[0,1]$ and let 
$r_l(\varepsilon)$ be the minimal SVD rank per layer achieving
$e_l\!\big(r_l(\varepsilon)\big)\le\varepsilon$ (by EYM).
Set
\begin{align}
H(\varepsilon)&:=\sum_{l=1}^L P_l\!\big(r_l(\varepsilon)\big),\nonumber\\
A(\varepsilon)&:=\sum_{l=1}^L \alpha_l\,e_l\!\big(r_l(\varepsilon)\big).\nonumber
\end{align}

By the Rank–$\varepsilon$ Allocation Equivalence (Prop.~\ref{prop:rank_eps_equiv}),
the constrained surrogate problem “minimize $\sum_l\alpha_l e_l(r_l)$
subject to $\sum_l P_l(r_l)\le b$” is equivalent to its $\varepsilon$-formulation.
Under the minimax-robust surrogate and bounded-profile condition
(Lemma~\ref{th:uniform_eps}), the optimal solution of this surrogate
for any budget $b$ is attained by a \emph{uniform} tolerance across layers.
Choosing $b:=H(\varepsilon)$, the uniform choice at level $\varepsilon$
is therefore optimal among all feasible allocations with total parameters
$\le H(\varepsilon)$, i.e., there is no other feasible allocation that
simultaneously achieves $\sum_l P_l(r_l)\le H(\varepsilon)$ and
$\sum_l \alpha_l e_l(r_l)<A(\varepsilon)$. Hence
$\big(H(\varepsilon),A(\varepsilon)\big)$ is (surrogate) Pareto-optimal for (B).

If $\varepsilon,\varepsilon'$ induce the same rank vector
$\big(r_l(\varepsilon)\big)_l=\big(r_l(\varepsilon')\big)_l$, then
$H(\varepsilon)=H(\varepsilon')$ and
$A(\varepsilon)=A(\varepsilon')$ by definition, so they correspond to the same
frontier point.
\end{proof}

\section{Derivation of ALS Updates}
\label{appx:als_derivation}
We minimize
\begin{align}
\label{eq:als-obj}
\min_{\mat{A},\mat{B}}
   \;\; \|\mat{W}\mat{X}-\mat{A}\mat{B}\mat{X}\|_F^2.
\end{align}
Let $\mat{M}=\mat{X}\mat{X}^\top$ and using the Frobenius identity
$\|\mat{Y}\|_F^2=\mathrm{tr}(\mat{Y}\mat{Y}^\top)$, we write
\begin{align}
&\mathrm{tr}\!\left((\mat{W}-\mat{A}\mat{B})\mat{M}
(\mat{W}-\mat{A}\mat{B})^\top\right) \nonumber \\
&= \mathrm{tr}(\mat{W}\mat{M}\mat{W}^\top)
 -2\,\mathrm{tr}(\mat{A}\mat{B}\mat{M}\mat{W}^\top) \nonumber \\
&\quad + \mathrm{tr}(\mat{A}\mat{B}\mat{M}\mat{B}^\top\mat{A}^\top). \nonumber
\end{align}

\paragraph{Update for $\mat{A}$.}
\begin{align}
\frac{\partial}{\partial \mat{A}}
 = -2\mat{B}\mat{M}\mat{W}^\top
   +2\mat{A}\mat{B}\mat{M}\mat{B}^\top. \nonumber
\end{align}
Setting to zero:
\begin{align}
\mat{A}
 = \mat{W}\mat{M}\mat{B}^\top
   (\mat{B}\mat{M}\mat{B}^\top)^{\dagger}. \nonumber
\end{align}

\paragraph{Update for $\mat{B}$.}
\begin{align}
\frac{\partial}{\partial \mat{B}}
 = -2\mat{A}^\top\mat{W}\mat{M}
   +2\mat{A}^\top\mat{A}\mat{B}\mat{M}. \nonumber
\end{align}
Setting to zero:
\begin{align}
\mat{B}
 = (\mat{A}^\top\mat{A})^{\dagger}\mat{A}^\top\mat{W}. \nonumber
\end{align}

\paragraph{ALS Updates.}
\begin{align}
\boxed{
\begin{aligned}
\mat{A}&=\mat{W}\mat{M}\mat{B}^\top
          (\mat{B}\mat{M}\mat{B}^\top)^{\dagger}, \\
\mat{B}&=(\mat{A}^\top\mat{A})^{\dagger}\mat{A}^\top\mat{W}.
\end{aligned}}
\end{align}

\begin{figure*}[t]
  \centering
  \begin{minipage}{0.5\linewidth}
    \centering
    \includegraphics[width=\linewidth]{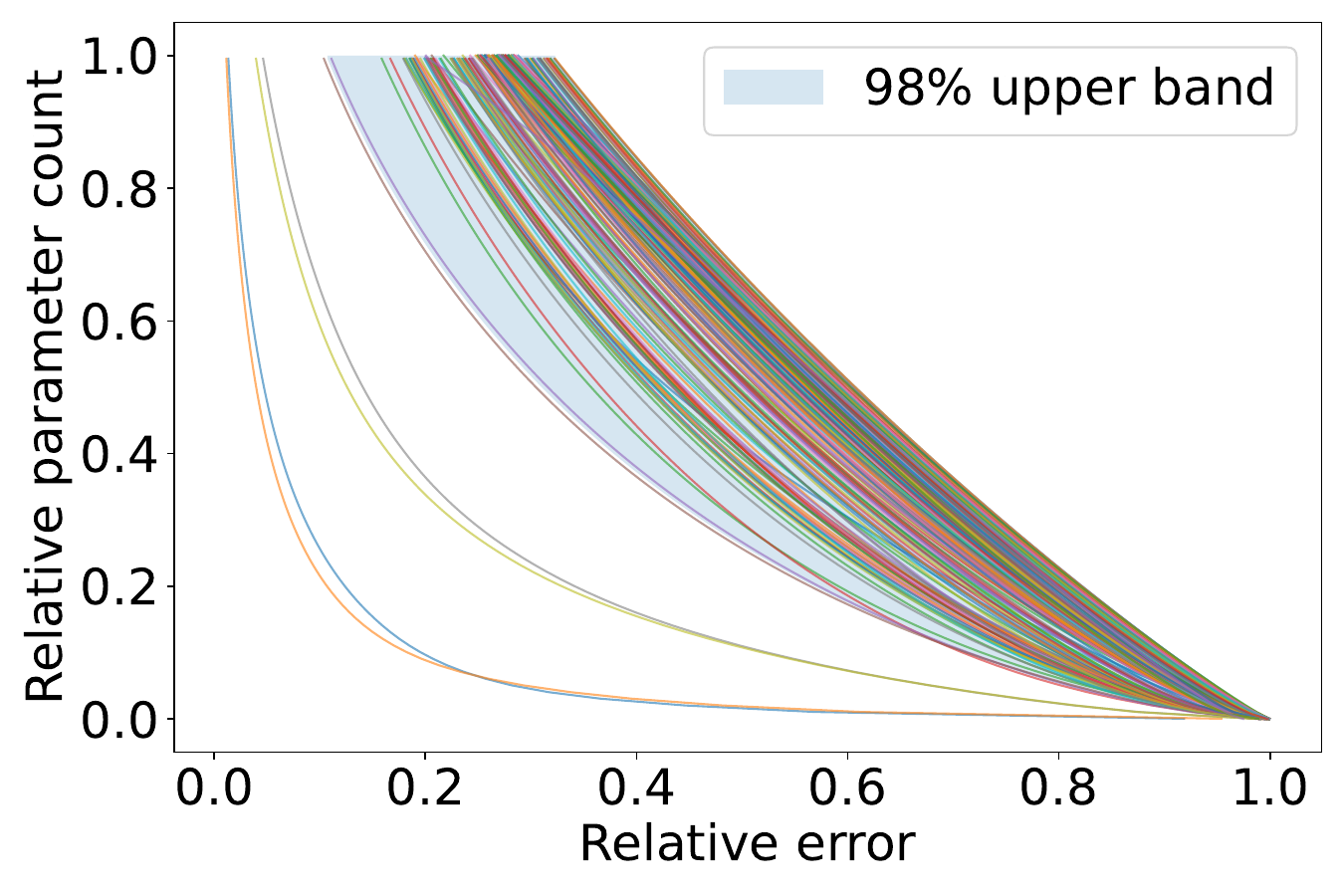}
  \end{minipage}\hfill
  \begin{minipage}{0.5\linewidth}
    \centering
    \includegraphics[width=\linewidth]{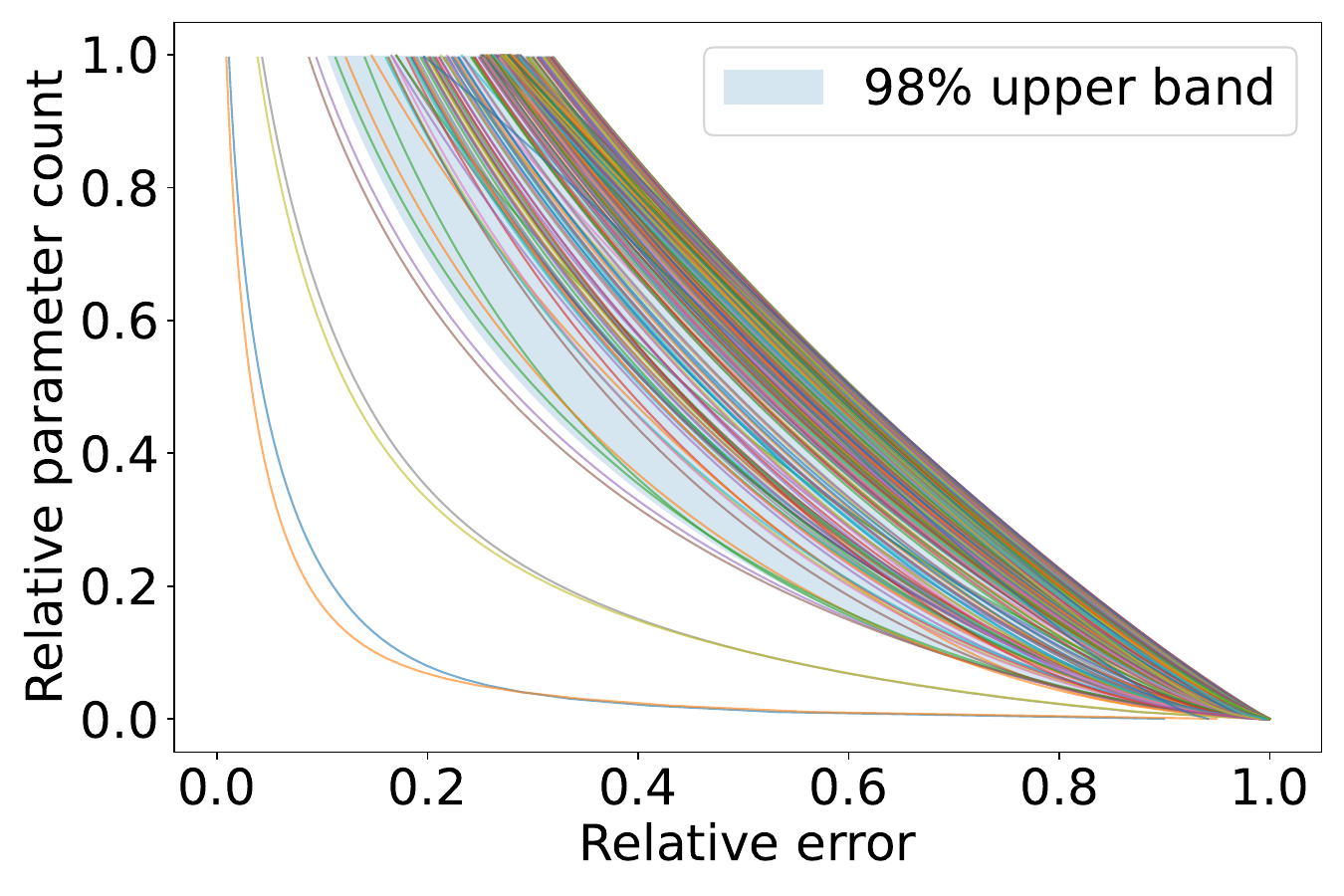}
  \end{minipage}
  \captionsetup{aboveskip=0pt}
  \caption{SVD profiles for LLaMA-2 7B (left) and 13B (right).}
  \label{fig:svd_profile}
\end{figure*}

\begin{table*}[t]
\scriptsize
\centering
\setlength{\tabcolsep}{4pt}
\renewcommand{\arraystretch}{1.15}
\caption{Zero-shot accuracy across tasks for LLaMA-2-7B and Mistral-7B at 20\% and 40\% compression (bold indicates the best accuracy per task among compressed methods).}
\label{tab:additional_reasoning}
\begin{tabular}{lll|cccccc|c}
\toprule
\textbf{Model} & \textbf{Compression} & \textbf{Method} 
& \textbf{ARC-C} & \textbf{BoolQ} & \textbf{HellaSwag} & \textbf{MathQA} & \textbf{MMLU} & \textbf{Race} & \textbf{Avg} \\
\midrule
{LLaMA-2-7B} & {} & {\color{gray}Base} & {\color{gray}43.34} & {\color{gray}77.71} & {\color{gray}76.00} & {\color{gray}28.07} & {\color{gray}41.82} & {\color{gray}39.52} & {\color{gray}51.08} \\
\cmidrule(lr){2-10}
 & 20\% & SVD-LLM   & 34.39 & 63.09 & 59.19 & 25.33 & 26.20 & 34.55 & 40.46 \\
 &      & SVD-ALS   & 33.70 & \textbf{68.01} & 59.55 & 24.96 & 26.64 & 33.97 & 41.14 \\
 &      & PGSVD & \textbf{36.52} & 67.43 & \textbf{60.96} & \textbf{25.76} & \textbf{27.48} & \textbf{35.02} & \textbf{42.20} \\
\cmidrule(lr){2-10}
 & 40\% & SVD-LLM   & 22.95 & 43.21 & 41.73 & 22.24 & 23.34 & \textbf{28.80} & 30.83 \\
 &      & SVD-ALS   & 22.87 & 42.57 & 41.74 & 22.38 & \textbf{23.54} & 27.85 & 30.16 \\
 &      & PGSVD & \textbf{23.63} & \textbf{56.12} & \textbf{43.42} & \textbf{22.65} & 23.29 & 28.61 & \textbf{32.62} \\
\midrule
{Mistral-7B} & {} & {\color{gray}Base} & {\color{gray}50.43} & {\color{gray}83.61} & {\color{gray}81.08} & {\color{gray}35.71} & {\color{gray}59.64} & {\color{gray}40.86} & {\color{gray}58.56} \\
\cmidrule(lr){2-10}
 & 20\% & SVD-LLM   & 35.15 & \textbf{67.74} & 59.60 & 28.54 & \textbf{28.53} & 35.50 & 42.51 \\
 &      & SVD-ALS   & 36.35 & 65.81 & 59.74 & 28.38 & 27.99 & 35.50 & 42.30 \\
 &      & PGSVD & \textbf{37.71} & 66.21 & \textbf{62.26} & \textbf{29.25} & 28.50 & \textbf{36.94} & \textbf{43.48} \\
\cmidrule(lr){2-10}
 & 40\% & SVD-LLM   & \textbf{21.67} & 38.62 & 37.87 & 22.18 & 23.38 & 27.08 & 28.47 \\
 &      & SVD-ALS   & 20.82 & 38.41 & 37.87 & \textbf{22.75} & 23.28 & 27.18 & 28.39 \\
 &      & PGSVD & 21.42 & \textbf{39.02} & \textbf{39.12} & 22.35 & \textbf{23.44} & \textbf{27.46} & \textbf{28.80} \\
\bottomrule
\end{tabular}
\end{table*}

\section{Convex Envelope of SVD Profiles} 
\label{appx:svd_profile_env}
The derivation of Pareto-guided ranks assumes that the SVD profiles (i.e., the rank–error relationships) are convexly bounded (Lemma~\ref{th:uniform_eps}).
Figure~\ref{fig:svd_profile} empirically supports this assumption. In Fig.~\ref{fig:svd_profile}, every single line corresponds to the SVD profile ($h_l$) of a specific layer $l$ of each network for all layers. It is observed that the relative SVD profiles naturally bounded by a convex upper bound. Although the uniform error-allocation strategy in PGSVD relies on the upper convex bound of the SVD profiles, the lower bound provides an indication of how close the achieved solution may be to the true optimum. As shown in Fig.~\ref{fig:svd_profile}, after excluding a small fraction of outlier layers, the lower bound of the remaining 98\% of layers approach the upper bound, suggesting that the practical rank–error distributions closely follow the assumed convex envelope.
\section{Calibration and Compression Cost}
\label{app:calibration}

\paragraph{Calibration.}
For LLMs, we follow the SVD-LLM protocol~\citep{svdllm_wang2024} and compute the covariance from 256 samples of WikiText-2, with compression and evaluation performed in a zero-shot setting. For CLIP, we use the full training split of each dataset for calibration and for selecting the error tolerances, while the test split is used for final evaluation.

\paragraph{Compression cost.}
The layer covariances add no overhead to the final compressed model. During compression, they require 0.45\,GB of transient memory and 28\,GB of storage on disk, and calibration takes 19 minutes on NVIDIA RTX 6000 Ada Generation GPUs for LLaMA-2-7B. We use full precision for model inference when computing the
covariance matrices and double precision during factorization. Compression times of the different solvers are reported in Fig.~\ref{fig:als-ablation}.

\section{Reasoning Tasks.} 
\label{app:reas_tasks}
For downstream evaluation, we used the LM Evaluation Harness and studied the performance of the models on the benchmarks ARC-Easy~\citep{clark2018arc}, CommonsenseQA~\citep{talmor2019commonsenseqa}, PIQA~\citep{bisk2020piqa}, RACE~\citep{lai2017race}, Winogrande~\citep{sakaguchi2020winogrande} and LAMBADA-Standard~\citep{paperno2016lambada} listed in Table~\ref{tab:language_results}. Table~\ref{tab:additional_reasoning} reported additional results on other reasoning tasks including ARC-Challenge~\citep{clark2018arc}, BoolQ~\citep{clark2019boolq}, MMLU~\citep{hendrycks2021mmlu}, HellaSwag~\citep{zellers2019hellaswag}, MathQA~\citep{amini2019mathqa}, and Race~\citep{lai2017race}.

\section{Comparison with ASVD}
\label{app:asvd}

We compare PGSVD with ASVD~\citep{yuan2023asvd}, which allocates heterogeneous
per-layer ranks via a sensitivity-based truncation rank searching (STRS). Table~\ref{tab:asvd}
reports WikiText-2 perplexity at a matched 20\% parameter reduction.

\begin{table}[h]
\caption{WikiText-2 perplexity of ASVD and PGSVD at 20\% compression.}
\centering
\small
\begin{tabular}{lccc}
\toprule
Method & LLaMA-2-7B & LLaMA-2-13B & Mistral-7B \\
\midrule
ASVD  & 21.70 & 8.07 & 13.47 \\
PGSVD & \textbf{7.38} & \textbf{5.96} & \textbf{6.71} \\
\bottomrule
\end{tabular}
\label{tab:asvd}
\end{table}

PGSVD achieves substantially lower perplexity across all models. Moreover, STRS
is a greedy search that requires about 1.5 GPU-hours for LLaMA-2-7B and
Mistral-7B and 4.5 GPU-hours for LLaMA-2-13B, whereas the uniform-$\varepsilon$
allocation in PGSVD is closed form and requires no search.

\end{document}